\documentclass{article} % For LaTeX2e
\usepackage{xr-hyper}
\usepackage[hypertexnames=false]{hyperref}
\usepackage{geometry}
\usepackage{times}
\usepackage{url}
\usepackage{subcaption}
\usepackage[parfill]{parskip}
\usepackage{mdl}

\usepackage{cleveref}
\usepackage{autonum}
\usepackage[square,numbers,sort]{natbib}

\title{Fairness-Aware Learning with Restriction of Universal Dependency using $f$-Divergences}
\author{
  Kazuto Fukuchi\\
  \texttt{kazuto@mdl.cs.tsukuba.ac.jp}\\
  Graduate School of SIE, University of Tsukuba\\
  \and
  Jun Sakuma\\
  \texttt{jun@cs.tsukuba.ac.jp}\\
  Graduate School of SIE, University of Tsukuba / JST CREST
}
\makeatletter
\def\@maketitle{%
  \newpage
  \null
  \vskip 2em%
  \begin{center}%
  \let \footnote \thanks
    {\LARGE \@title \par}%
    \vskip 1.5em%
    {\large
      \begin{tabular}[t]{c}%
        \@author
      \end{tabular}\par}%
  \end{center}%
  \par
  \vskip 1.5em}
\makeatother

\begin{document}

\maketitle

\begin{abstract}
\noindent
{\em Fairness-aware learning} is a novel framework for classification tasks. Like regular empirical risk minimization~(ERM), it aims to learn a classifier with a low error rate, and at the same time, for the predictions of the classifier to be independent of sensitive features, such as gender, religion, race, and ethnicity.
Existing methods can achieve low dependencies on given samples, but this is not guaranteed on unseen samples.
The existing fairness-aware learning algorithms employ different dependency measures, and each algorithm is specifically designed for a particular one. Such diversity makes it difficult to theoretically analyze and compare them. In this paper, we propose a general framework for fairness-aware learning that uses $f$-divergences and that covers most of the dependency measures employed in the existing methods. We introduce a way to estimate the $f$-divergences that allows us to give a unified analysis for the upper bound of the estimation error; this bound is tighter than that of the existing convergence rate analysis of the divergence estimation.
With our divergence estimate, we propose a fairness-aware learning algorithm, and perform a theoretical analysis of its generalization error.
Our analysis reveals that, under mild assumptions and even with enforcement of fairness, the generalization error of our method is $O(\sqrt{1/n})$, which is the same as that of the regular ERM. In addition, and more importantly, we show that, for any $f$-divergence, the upper bound of the estimation error of the divergence is $O(\sqrt{1/n})$. This indicates that our fairness-aware learning algorithm guarantees low dependencies on unseen samples for any dependency measure represented by an $f$-divergence.
\end{abstract}

\section{Introduction}
Recently developed information systems are being increasingly incorporating machine learning techniques for making important decisions, such as credit scoring, calculating insurance rates, and evaluating employment applications.
These decisions can result in the unfair treatment, if the decisions depend on the sensitive information, such as the individual's gender, religion, race, or ethnicity.
Fairness-aware learning attempts to solve this problem, and has recently received a great deal of attention~\citep{pedreshi2008discrimination,Calders:2010:TNB:1842547.1842562,Sweeney:2013:DOA:2460276.2460278}.
In this paper, we consider the use of fairness-aware learning for classification problems.

Let $\dom{X}$ and $\dom{Y} = \cbrace{1,...,c}$ be the domain of the input and the domain of the target, respectively.
In ordinary classification algorithms, the learner aims to find a hypothesis $f: \dom{X}\to\dom{Y}$ that minimizes misclassifications from a given set of iid samples $\set{S}_n = \cbrace{ (x_i, y_i) }_{i=1}^n \in (\dom{Z} = \domprod{XY})^n$. In fairness-aware learning, we assume that the input $x_i$ contains a {\em viewpoint} $v_i \in \dom{V}$, which represents the sensitive information of individuals. The learner aims to find an $f$ that will have a low misclassification rate and for which the output of $f$ has little dependency on the viewpoint $v$.
For example, suppose the company want to make a hiring decision using information collected from job applicants~(input $x$), including their age, place of residence, and work experience, but also including their gender, religion, race, and ethnicity~(viewpoint $v$).
We wish to make hiring decisions based on the potential work performance of the job applicants~(target $y$) via a supervised learning algorithm, $y = f(x)$.
We say $f$ is discriminatory if the output of $f$ is dependent on the viewpoint $v$~\citep{pedreschi2009measuring}. Fairness-aware learning attempts to avoid such unfair decisions by minimizing the dependency of the output of $f$ on the viewpoint $v$~\citep{Calders:2010:TNB:1842547.1842562,kamiran2010discrimination,kamishima.ecml.pkdd2012.fairness-aware,DBLP:conf/icml/ZemelWSPD13,DBLP:conf/pkdd/FukuchiS14}.
Needless to say, minimization of the misclassification rate and minimization of the dependency are conflicting targets.
Therefore, we need to consider the trade-off between  misclassification and dependency.

The existing methods resolve this conflict by suppressing the dependency to sensitive viewpoints; this is accomplished by introducing a regularization term~\citep{kamishima.ecml.pkdd2012.fairness-aware,DBLP:conf/icml/ZemelWSPD13,DBLP:conf/pkdd/FukuchiS14} or by adding constraints~\citep{Calders:2010:TNB:1842547.1842562,kamiran2010discrimination,zliobaite2011handling}, to the objective function of the regular empirical risk minimization~(See \cref{sec:related-works}). Typically, such techniques can lead to predictions for which there is less dependency on the sensitive viewpoints of the given samples~(empirical dependency). However, predictors with low empirical dependency do not necessarily achieve low dependency on sensitive viewpoints of unseen samples~(generalization dependency). In the hiring decision example, the hypothesis $f$ is trained with information collected from the past histories of job applicants. Predictors trained with existing methods might make fair decisions for the job applicants in the past~(low empirical dependency). However, fair decisions for the job applicants in the future~(low generalization dependency) are not guaranteed. Except for the method of \citet{DBLP:conf/pkdd/FukuchiS14}, most of the existing methods have no theoretical guarantee of the generalization dependency. In \citep{DBLP:conf/pkdd/FukuchiS14}, theoretical analysis provides a probabilistic bound on the generalization dependency, but the analysis is derived for only a specific measure of dependency.

\noindent{\bfseries Our contributions.}

We perform a unified analysis of the fairness-aware learning with more general dependency measures based on the $f$-divergence~\citep{ali1966general,csiszar63}. The $f$-divergence is a universal class of the divergences, which can represent most of existing divergences, including the total variational distance, the covariance, the Hellinger distance, the $\chi^2$-divergence, and the KL-divergence. Our fairness-aware learning basically follows the framework of empirical risk minimization~(ERM).
The goal of fairness-aware learning is to obtain predictors with an upper bound guarantee of generalization dependency; however, it cannot be directly evaluated because the underlying distribution is not observable. We thus derive an upper bound of the generalization dependency by the empirical dependency plus two extra terms. Our framework achieves the fairness of the resultant predictors by restricting the class of hypotheses to those with low empirical dependency. Thus, the upper bound of the generalization dependency of the predictors can be theoretically derived by using the bound.

The contributions of this study are two-fold.
First, we propose a novel generalized procedure for estimating the $f$-divergences for fairness-aware learning. Our estimation method can be regarded as a generalization of \citep{nguyen2010estimating,ruderman2012tighter,kanamori2009least}. As already stated, we constrain the hypothesis class by the $f$-divergence for guarantee of fairness. It is thus important to derive a tighter upper bound of the $f$-divergence to achieve lower generalization dependency. Existing divergence estimation methods~\cite{nguyen2010estimating} provides an upper bound of the $f$-divergence; however, the bound is not suitable for our purpose for the following two reasons. First, their analysis is specifically derived for KL-divergence and cannot be expanded to the general $f$-divergences. Second, their bound is derived for convergence analysis, not for the upper bound of the divergence. Thus, the bound is loose for our purpose.
Our generalized estimation procedure provides a tighter upper bound of the $f$-divergences by introducing the maximum mean discrepancy. % to the estimation procedure.
As a result, the estimation error of the $f$-divergence is bounded above by the empirical maximum mean discrepancy and by $O(\sqrt{1/n})$.

Second, we formulate a general ERM framework for fairness-aware learning with employing the $f$-divergence. % as the dependency measure.
We analyze the generalization error and generalization dependency of the proposed fairness-aware learning algorithm, and we show that even when fairness is enforced, the generalization error can be bounded above by the Rademacher complexity and $O(\sqrt{1/n})$, as in the regular ERM.
The generalization dependency can be bounded above by the empirical maximum mean discrepancy term and two other extra terms. Thanks to the theoretical analysis of generalization dependency, we can theoretically compare the upper bound of the estimation error by dependency measures. Our analysis revealed that the divergence estimation errors for all of these divergences are $O(\sqrt{1/n})$ equally, and the Hellinger distance achieves the lowest estimation error in terms of the constant term of the probabilistic error. We also derived a convex formulation of fairness-aware learning that works with any dependency measures represented by the $f$-divergence. The optimization problem can be readily solved by a standard convex optimization solver.

\section{Related Works}
\label{sec:related-works}
Within the setup described in the Introduction, \citet{Calders:2010:TNB:1842547.1842562} pointed out that elimination of the viewpoint from the given samples is insufficient for achieving low correlation between the output of $f$ and the viewpoint; this is because the viewpoint has an indirect influence since it is not independent from the input. For example, when we make a hiring decision using information collected from job applicants via supervised learning, even if we train with samples that exclude race and ethnicity, the output of the resultant hypothesis $f$ may be indirectly correlated with race or ethnicity, because the addresses of the applicants may be correlated with their race or ethnicity. Such an indirect effect is called the red-lining effect~\citep{Calders:2010:TNB:1842547.1842562}.

To remove the red-lining effect, existing works have attempted to construct a classifier that results in a fairer hypothesis. \citet{Calders:2010:TNB:1842547.1842562} proposed the naive Bayes classifier with fairness constraint, which employs the difference between the conditional probabilities $\abs*{\p(f(\set{X}) = y_+ | v_+) - \p(f(\set{X}) = y_+ | v_-)}$ where $\dom{Y} = \cbrace{y_+, y_-}$ and $\dom{V} = \cbrace{v_+, v_-}$.
\citet{kamiran2010discrimination} and \citet{zliobaite2011handling} discussed various situations in which discrimination can occur, in terms of the difference of conditional probabilities. \citet{dwork2012fairness} introduced a fairness-aware learning framework of the ERM with constraints of {\em statistical parity}, defined as the total variational distance between $\p(f(\set{X})|v_+)$ and $\p(f(\set{X})|v_-)$. \citet{DBLP:conf/icml/ZemelWSPD13} presented an algorithm to preserve fairness in a classification setting based on statistical parity.
\citet{kamishima.ecml.pkdd2012.fairness-aware} proposed a fairness-aware learning algorithm of the maximum likelihood estimation by penalizing the log-likelihood using KL-divergence between $\p(f(\set{X}),\set{V})$ and $\p(f(\set{X}))\p(\set{V})$\footnote{The KL divergence between $\p(f(\set{X}),\set{V})$ and $\p(f(\set{X}))\p(\set{V})$ is known as mutual information between $f(\set{X})$ and $\set{V}$.}.
These fairness-aware learning algorithms do not have a theoretical guarantee for the estimation error of the dependency measures. In addition, the design of these algorithms are tightly coupled with specific dependency measures. They thus have less flexibility for choosing other dependency measures.

The fairness for unseen samples can be measured by the estimation error bound of the dependency measure.
\citet{DBLP:conf/pkdd/FukuchiS14} first derived a bound on the estimation error of a specific measure, namely the {\em +1/-1 neutrality risk}. They proved that the estimation error of the measure is bounded above in probability by the Rademacher complexity of the hypothesis class $\fset{F}$ and $O(\sqrt{1/n})$ term. Unfortunately, the analysis relies on the +1/-1 neutrality risk and cannot be generalized to other types of dependency measures.
%In this paper, we provide a framework for fairness-aware learning that works with any dependency measure represented by $f$-divergence. In addition, our framework provides the estimation error bound of dependency for any measure represented by $f$-divergence, which allows to theoretically compare the estimation error bounds of dependency realized by various divergence measures.

Estimation procedures that use $f$-divergences that are based on iid samples have been studied extensively. For example, for the KL-divergence, method have been proposed that use nearest-neighbor distances~\citep{wang2009divergence} and least-squares estimations of the probability ratio~\citep{kanamori2009least}. To estimate $f$-divergences, \citet{DBLP:conf/icml/Garcia-GarciaLS11} introduced an estimation procedure that uses loss minimization and sampling. \citet{kanamori2012divergence} presented a divergence estimator of the $f$-divergences based on using the moment matching estimator~\citep{qin1998inferences} to estimate the probability ratio. \citet{nguyen2010estimating} used a property of convex conjugate functions to derive the M-estimator of $f$-divergences, and they also derived its convergence rate.
%The interest of the theoretical analysis of these estimation methods are on the convergence rate of the absolute value of the estimation error. On the other hand, our interest is on the upper bounds on the estimation error, which is needed for analysis of fairness-aware learning.
In our analysis, we derive the upper bound of the estimation error, which yields a tighter upper bound of the estimation error of dependency measures compared to the existing convergence rate analysis.

\section{Problem Formulation}
Let $\dom{X}$ and $\dom{Y} = \cbrace{1,...,c}$ be the domain of the input and the domain of the target, respectively. We assume that the learner obtains a set of iid samples $\set{S}_n = \cbrace{ (x_i, y_i) }_{i=1}^n \in (\dom{Z} = \domprod{XY})^n$ that are drawn from an unknown probability measure $\mu$, which is defined on some measurable space $(\dom{Z}, \family{Z})$. In addition, we assume the input $x_i$ consists of the viewpoint $v_i \in \dom{V}$ and various other features $w_i \in \dom{W}$. Thus, $\dom{X} = \dom{V}\times\dom{W}$.
Given iid samples, the learner seeks to find a hypothesis $f: \dom{X}\to\dom{Y}$ from a class of measurable functions $\fset{F}$ that minimizes both the misclassification rate and the dependence on the viewpoint $v$. We denote $(X_i, Y_i)$ for $i = 1,...,n$ as the random variables of the samples, and we denote $V_i$ as the random variable for the corresponding viewpoint.

The misclassification of the hypothesis $f$ is evaluated by the generalization risk, which is defined as $R(f) = \Mean[\ell(Y, f(X))]$.
%In a classification problem, t
The goal of the learner is to find the hypothesis $f^* \in \fset{F}$ such that
\begin{align}
   R(f^*) = \inf_{f \in \fset{F}} R(f).
\end{align}
The generalization risk $R(f)$ cannot be evaluated directly because the sample distribution $\mu$ is unknown. Instead of the generalization risk, empirical risk minimization~(ERM) finds a hypothesis $f_n \in \fset{F}$ that minimizes the empirical risk
\begin{align}
  R_n(f) = \frac{1}{n}\sum_{i=1}^n \ell(Y_i, f(X_i)).
\end{align}
Minimization of the empirical risk results in a relatively low generalization risk, and the generalization risk of the resultant hypothesis converges towards that of the optimal hypothesis $f^*$ as the number of samples increases; this has been shown theoretically~\citep{bartlett2005local}.

\subsection{A Generalized Class for Dependency Measures}

For the evaluation of the dependency of the output of $f$ on the viewpoint $v$, we define a general class of measures for dependency.% that includes all of the various known dependency measures.
Given that $\pidep{V,f(X)} = \pjoin{V,f(X)}$ if $f(X)$ and $V$ are statistically independent, we can evaluate this dependency by evaluating the difference between the two probability measures, $\pidep{V,f(X)}$ and $\pjoin{V,f(X)}$. To measure the difference between two probability measures, we use the $f$-divergences.
Suppose $\pmet{P}$ and $\pmet{Q}$ are two probability measures on a compact domain $\dom{X}$, where $\pmet{P}$ is absolutely continuous with respect to $\pmet{Q}$. The class of $f$-divergences, also known as the Ali--Silvey distances~\citep{ali1966general,csiszar63}, takes the form
\begin{align}
 D_\phi(\pmet{P}, \pmet{Q}) =& \int \phi\paren*{\frac{d\pmet{P}}{d\pmet{Q}}} d\pmet{Q},
\end{align}
where $\phi: \RealSet^+\to\RealSet$ is a convex and lower semicontinuous function such that $\phi(1) = 0$.\footnote{The $f$-divergence becomes one of the existing divergences due to the choice of $\phi$; that is, it becomes the total variational distance if $\phi(u) = \abs{u-1}$, the Hellinger distance if $\phi(u) = (\sqrt{u} - 1)^2$, the $\chi^2$-divergence if $\phi(u) = (u-1)^2/u$, or the KL-divergence if $\phi(u) = (u-1)-\ln(u)$. See \cref{fig:phi}.}
After we define the $f$-divergences, we define the measure of the dependency between $f(X)$ and $V$ as follows:
\begin{align}
  D_\phi(f) =& D_\phi(\pidep{V,f(X)}, \pjoin{V,f(X)}).
\end{align}
Without loss of generality, we will assume that the subdifferential of $\phi$ at $1$ contains $0$. This can be readily confirmed by $\phi_c(u) = \phi(u) - c(u-1)$ which does not change the value of the $f$-divergences $D_{\phi_c}(\pmet{P},\pmet{Q}) = D_\phi(\pmet{P},\pmet{Q})$ for any finite $c$. We will focus on the convex functions $\phi$ that are differentiable on $\RealSet^+$ except at $1$. Note that this includes most of the divergences, including the total variational distance, the Hellinger distance, the $\chi^2$-divergence, and the KL-divergence.

\subsection{Fairness-Aware Learning with Generalized Dependency Measures}

In fairness-aware learning, the learner attempts to minimize both $R(f)$ and $D_\phi(f)$.
However, since $\argmin_{f \in \fset{F}}R(f) = \argmin_{f\in\fset{F}}D_\phi(f)$ does not always hold, there exists a trade-off between $R(f)$ and $D_\phi(f)$.
We thus consider a subset of $\fset{F}$ parameterized by $\eta \ge 0$ defined as follows:
\begin{align}
  \fset{F}_\eta = \cbrace*{f \in \fset{F} \middle| D_\phi(f) \le \eta }.
\end{align}
Thus, the goal of the fairness-aware learning is to achieve the hypothesis $f_\eta^* \in \fset{F}_\eta$ that satisfies
\begin{align}
  R(f_\eta^*) = \inf_{f \in \fset{F}_\eta} R(f).
\end{align}
Again, since the generalization risk cannot be evaluated directly, the learner minimizes the empirical risk as
\begin{align}
  \min_{f \in \fset{F}_\eta} R_n(f). \label{eq:erm-div-rest}
\end{align}
The objective of fairness-aware learning is to solve the optimization problem of \cref{eq:erm-div-rest}. Unfortunately, $D_\phi(f)$ cannot be evaluated directly again since the underlying distribution is unobservable. In \cref{sec:div-est}, we introduce a novel estimation procedure of the $f$-divergences to alleviate evaluation of $D_\phi(f)$. Then, we prove an upper bound of $D_\phi(f)$ with empirical estimation of $D_\phi(f)$ given a finite number of samples. In \cref{sec:fal-div-est}, the objective function of fairness-aware learning is redefined using the empirical estimation of $D_\phi(f)$.

\section{Divergence Estimation}
\label{sec:div-est}

In this section, we introduce a procedure that involves minimizing the {\em maximum mean discrepancy}~(MMD) for estimating $D_\phi(f)$, and we determine a non-asymptotic bound on the estimation error. This procedure covers the existing $f$-divergences or KL-divergence estimation algorithms proposed by \citet{nguyen2010estimating}, \citet{ruderman2012tighter}, and \citet{kanamori2012divergence}.

\subsection{Estimation of the Divergence by Minimizing the Maximum Mean Discrepancy}

To estimate $D_\phi(f)$, we first empirically estimate the probability ratio $r(V,f(X)) = d\pidep{V,f(X)}/d\pjoin{V,f(X)}$, and then we empirically evaluate $D_\phi(f)$ by using the estimated probability ratio.
Since $d\pidep{V,f(X)} = r(V,f(X))d\pjoin{V,f(X)}$ holds for the probability ratio $r$, the minimizer of the difference between $d\pidep{V,f(X)}$ and $r(V,f(X))d\pjoin{V,f(X)}$ is expected to be close to the probability ratio. As a measure of the disparity of $d\pidep{V,f(X)}$ and $r(V,f(X))d\pjoin{V,f(X)}$, we use the maximum mean discrepancy. Let $\fset{G}$ be a set of functions $g:\domprod{VY}\to\RealSet$. Let $X'$ be an independent copy of $X$, and let $V'$ be the viewpoint of $X'$. Then, the MMD with $\fset{G}$ between $d\pidep{V,f(X)}$ and $r(V,f(X))d\pjoin{V,f(X)}$ is defined as
\begin{align}
 D_{{\rm MMD}, f}(r)\!=&\!\sup_{g \in \fset{G}} \bracket*{\int g(V, f(X)) d(\pidep{V,f(X)})\!-\!\int g(V, f(X))r(V,f(X)) d\pjoin{V,f(X)} }\\
  \!=&\!\sup_{g \in \fset{G}} \bracket*{\Mean[g(V', f(X))]\!-\!\Mean[r(V,f(X))g(V, f(X))]}. \label{eq:mmd}
\end{align}
If $r$ is equivalent to the probability ratio, we have $D_{{\rm MMD}, f}(r)=0$.
However, $D_{{\rm MMD}, f}(r)=0 \implies d\pidep{V,f(X)} = r(V,f(X))d\pjoin{V,f(X)}$ does always not satisfy, which requires that $\fset{G}$ is a set of functions on a universal kernel~\citep{DBLP:journals/jmlr/Steinwart01}. Therefore, the evaluation ability of the discrepancy of MMD is dependent on the choice of $\fset{G}$.%, and thus we discuss the property of $\fset{G}$ for achieving the tighter upper bound of the estimation error in the next subsection.
The U-statistics~\citep{hoeffding:1948} gives an unbiased estimator of \cref{eq:mmd} as
\begin{align}
 D_{{\rm MMD}, f, n}(r) = \sup_{g \in \fset{G}} \bracket*{ \frac{1}{n(n-1)}\sum_{1 \le i \ne j \le n}\bracket*{ g(V_i, f(X_j)) - r(V_i, f(X_i))g(V_i, f(X_i)) } }. \label{eq:mmd-emp}
\end{align}
The estimator of the probability ratio $r_n$ is obtained by minimizing $D_{{\rm MMD}, f, n}(r)$.\footnote{The efficient computation of the empirical MMD is shown in \cref{sec:mmd-rkhs}.}
We can add the regularizer term $\Omega(r)$ to the empirical MMD to ensure the consistency of the estimator:
\begin{align}
 \min_{r_n \ge 0} ~ D_{{\rm MMD}, f, n}(r_n) + \Omega(r_n). \label{eq:est-ratio}
\end{align}
After obtained $r_n$ by solving \cref{eq:est-ratio}, the $f$-divergence is empirically evaluated as
\begin{align}
  D_{\phi,n}(f) = \frac{1}{n}\sum_{i = 1}^n \phi\paren*{r_n(V_i, f(X_i))}.
\end{align}
The estimation procedure is equivalent to \citep{nguyen2010estimating} if $\Omega(r) = \lambda_n D_{\phi,n}(f)$ where $\lambda_n$ is regularizer parameter. In addition to the regularizer term, if we add the constraint $\tfrac{1}{n}\sum_{i=1}^n r_n(V_i, f(X_i)) = 1$ into \cref{eq:est-ratio}, the estimation procedure becomes same as \citep{ruderman2012tighter}. Letting $\Omega(r) = 0$ and the appropriate choice of $\fset{G}$ yields the estimation procedure of \citep{kanamori2012divergence}.

\subsection{Analysis of Estimation Error}
In this subsection, we show the upper bound on the estimation error
\begin{align}
 D_\phi(f) - D_{\phi,n}(f).
\end{align}
Surprisingly, the upper bound of the estimation error does not depend on the complexity of the class of functions $\fset{G}$.
In what follows, we use $D_\phi^r(f) = \Mean[\phi(r(V, f(X)))]$ and $D_{\phi,n}^r(f) = \sum_{i=1}^n \phi(r(V_i, f(X_i)))/n$. In addition, we denotes the true probability ratio as $r^*(V,f(X))$.% = d\pidep{V,f(X)}/d\pjoin{V,f(X)}$.

The following theorem states the probabilistic upper bound on the estimation error.
\begin{theorem}\label{thm:div-bound}
  Let $r_n$ be the probability ratio estimated from the obtained set of samples $\set{S}_n$. Suppose that the class of the functions $\fset{G}$ of the MMD contains $\partial\phi(r^*(V, f(X))) / a$, where $\partial\phi$ is an element of the subdifferential of $\phi$, $a > 0$ is some constant, and $c_\ell \le r^*(V, f(X)) \le c_u$ almost surely, where $c_\ell \in (0, 1]$ and $c_u \in [1, \infty)$. Then, with probability at least $1 - e^{-t}$
  \begin{align}
    D_\phi(f) \le D_{\phi,n}(f) + aD_{{\rm MMD}, f, n}(r_n) + c\sqrt{\frac{2t}{n}},
  \end{align}
  where $c = 2\max\cbrace{\phi(c_\ell), \phi(c_u)} + \partial\phi(c_u)c_u + \partial\phi(c_u) - 2\partial\phi(c_\ell)$.
\end{theorem}
The proof of this theorem is found in \cref{sec:proofs}.
As proved in \cref{thm:div-bound}, the $f$-divergences can be bounded above by the empirical $f$-divergences, the empirical MMD, and $O(\sqrt{1/n})$. We minimize the error between the $f$-divergences and the empirical $f$-divergences by minimizing the empirical MMD. In addition, the error bound does not depend on the complexity of $\fset{G}$. This implies that in order to guarantee the upper bound on the $f$-divergences, we should choose $\fset{G}$ so that it is large enough to satisfy $\partial\phi(r^*(V, f(X))) \in \fset{G}$. A large $\fset{G}$, however, can lead an over estimation of the $f$-divergences.

The convergence rate, i.e., the absolute value of the estimation error, as shown by \citep{nguyen2010estimating} is dependent on the convergence rate of the empirical process with respect to $\fset{G}$. However, the upper bound proved by \cref{thm:div-bound} does not contain the complexity term of $\fset{G}$, such as Rademacher complexity, the covering entropy and the bracketing entropy, and thus is tighter than the convergence rate.

\section{Fairness-Aware Learning with a Divergence Estimation}
\label{sec:fal-div-est}

In this section, we provide an algorithm for solving \cref{eq:erm-div-rest} that includes the introduced estimation procedure for the $f$-divergences. We will then show that the algorithm can be formulated as a convex optimization problem.

\subsection{Algorithm for Fairness-Aware Learning with $f$-Divergence Estimation}

Following the estimation procedure described in \cref{sec:div-est}, we define the optimization problem of our fairness-aware learning as
\begin{align}
 \min_{f \in \fset{F}, r_n \ge 0} ~ R_n(f) ~ \subto ~ D_{\phi,n}(f) + a_n D_{{\rm MMD},f,n}(r_n) \le \eta, \label{eq:fal}
\end{align}
where $a_n$ is the constant larger than $a$ that was defined in \cref{thm:div-bound}.
As indicated by \cref{thm:div-bound}, $D_\phi(f) \le D_{\phi,n}(f) + a_n D_{{\rm MMD},f,n}(r_n) + O(\sqrt{1/n})$ holds, which guarantees that the $f$-divergence of the resultant hypothesis of \cref{eq:fal} is less than $\eta + O(\sqrt{1/n})$.

Let us consider the effect of the choice of $\phi$ on the estimation error of the divergence. The upper bound of the estimation error shown in \cref{thm:div-bound} does not depend on the choice of $\phi$. Nevertheless, the choice of $\phi$ changes the constant $c$. Letting $c_\ell = e^{-t}$ and $c_u = e^t$, \cref{fig:phi-c} shows the shape of $\phi$ and the value of $c$ corresponding to $t$ for various functions $\phi$. As shown in \cref{fig:const-c}, the smallest $c$ is that for the Hellinger distance, and thus of these four divergences, it has the tightest bound on the probability ratio $r^*$.

\begin{figure}[t]
 \centering
 \begin{subfigure}[b]{.45\linewidth}
  \includegraphics[width=\textwidth]{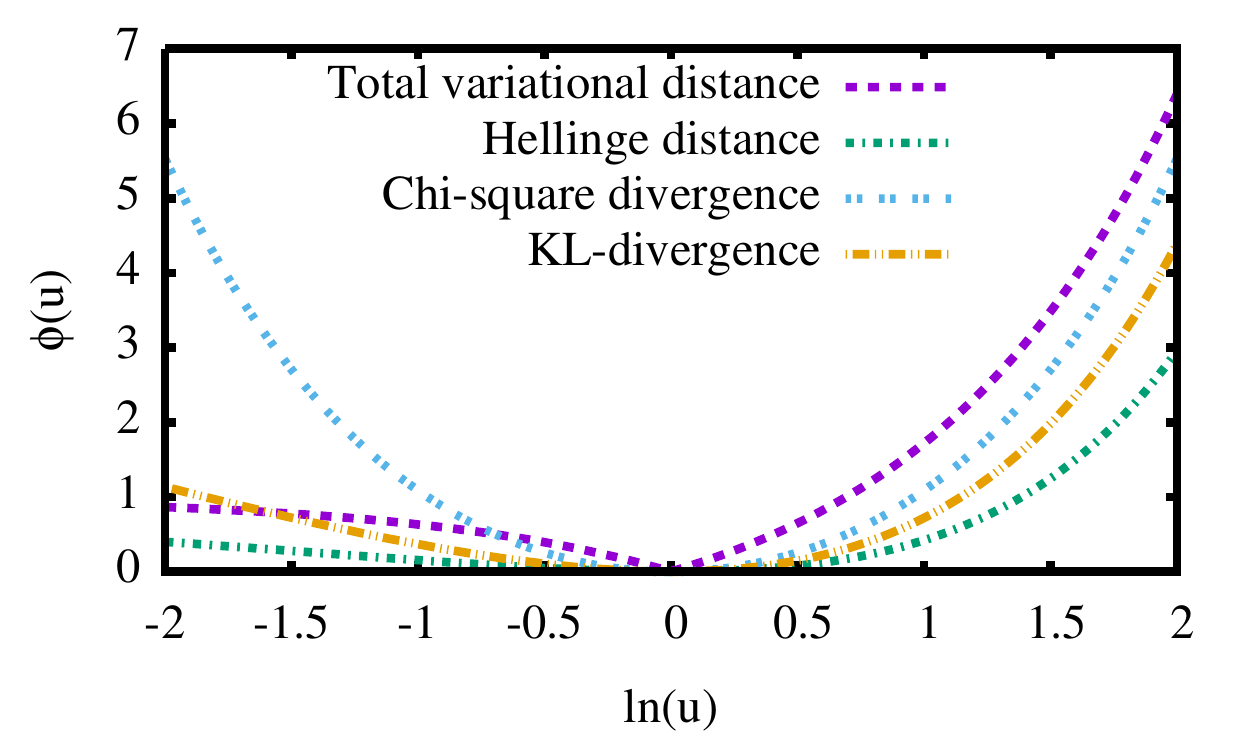}
  \caption{$\phi$}\label{fig:phi}
 \end{subfigure}%
 \begin{subfigure}[b]{.45\linewidth}
  \includegraphics[width=\textwidth]{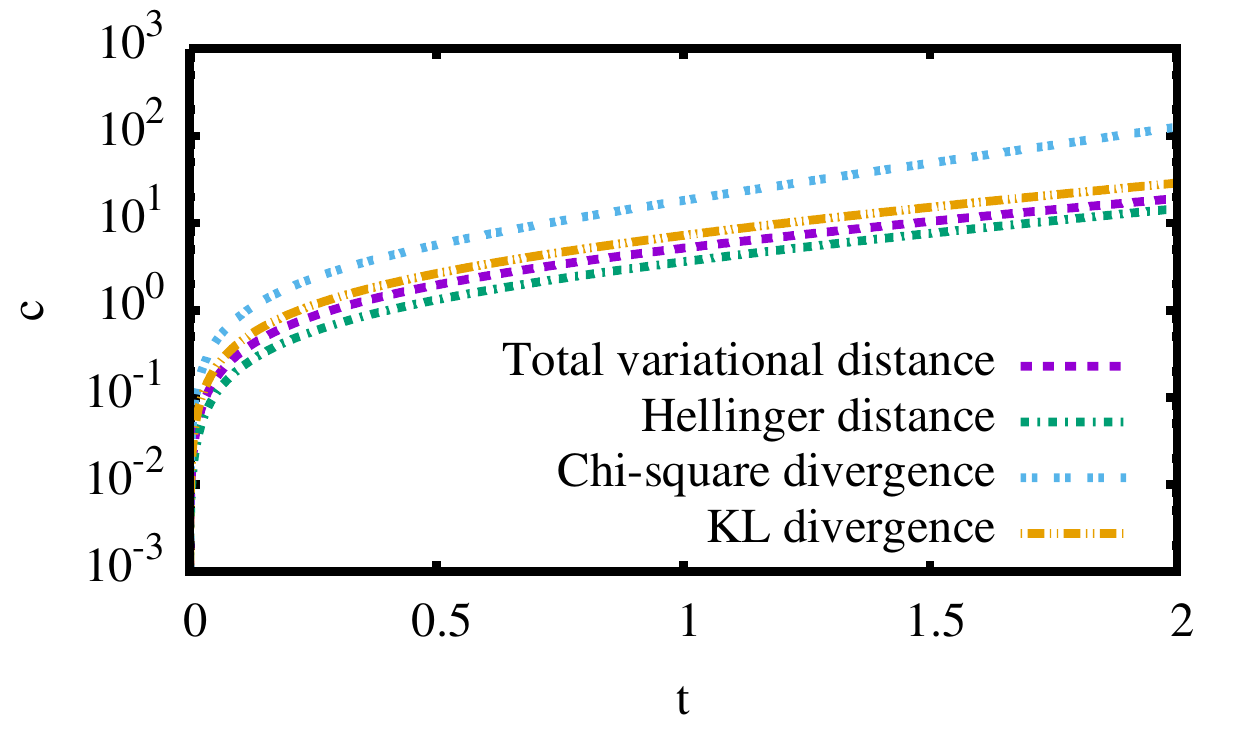}
  \caption{$c$}\label{fig:const-c}
 \end{subfigure}
 \caption{The shape of $\phi$ and the value of $c$ in \cref{thm:div-bound} for various $\phi$, where $c_\ell = e^{-t}$ and $c_u = e^t$.}\label{fig:phi-c}
\end{figure}

\subsection{Optimization}

%The necessary condition of convexity of \cref{eq:fal} with the RKHS version of the MMD is the linearity of the kernel function $k$ with respect to $f$. With mild assumptions, it can be made convex for any choice of the kernel function by a simple reformulation.
The necessary condition of convexity of \cref{eq:fal} is the linearity of the functions $g \in \fset{G}$ with respect to $f$. With mild assumptions, it can be made convex for any choice of $\fset{G}$ by a simple reformulation.
\begin{assumption}\label{asm:hypothesis}
 The hypothesis $f \in \fset{F}$ is formed as $f(x) = \argmax_{y \in \dom{Y}} \theta(x, y)$, and $\theta \in \Theta$ is linear with respect to the parameters.
\end{assumption}
With this assumption, the function on the RKHS $\mathcal{H}$ is given as $\theta(x, y) = \langle\Phi(x,y), \vec{w}\rangle_\mathcal{H}$, where $\Phi:\domprod{XY}\to\dom{H}$ and $\vec{w} \in \dom{H}$. The optimization problem in \cref{eq:fal} can be rearranged as
\begin{align}
 \min_{f \in \fset{F}}  ~ R_n(f) ~ \subto ~ \min_{r_n \ge 0} (D_{\phi,n}(f) + a_n D_{{\rm MMD},f,n}(r_n)) \le \eta. \label{eq:opt-gu}
\end{align}
Since $r_n$ is not appeared in the objective function in original optimization problem \cref{eq:fal}, we change the optimization problem so that the optimization with respect to $r_n$ is only appeared in the constraint.
Following the derivation of the dual problem in \citep{nguyen2010estimating}, we have the dual form of the constraint as
\begin{align}
 \min_{r_n \ge 0} (D_{\phi,n}(f) + a_n D_{{\rm MMD},f,n}(r_n)) = \max_{g \in \fset{G}}\paren*{ E_{\phi,f,n}(g) },
\end{align}
where
\begin{align}
 E_{\phi,n}(g) = \frac{1}{n(n-1)}\sum_{1 \le i \ne j \le n} g(v_j, f(x_i)) - \frac{1}{n}\sum_{i = 1}^n \phi^*(g(v_i, f(x_i))) + \frac{1}{2a_n}\norm{g}_{\dom{H}}^2,
\end{align}
and $\phi^*(v) = \sup_{u}(uv - \phi(u))$. Letting $\vec\gamma \in \RealSet^{n \times c}$, we can rewrite the optimization problem in \cref{eq:opt-gu} as
\begin{align}
 \min_{f \in \fset{F}, \vec\gamma \in \RealSet^{n\times c}} ~ R_n(f) ~
\subto ~ I[f(x_i) = j] = \gamma_{ij} ~ \forall i, j, ~ \max_{g \in \fset{G}} E_{\phi,n}(\vec\gamma, g) \le \eta, \label{eq:opt-eq}
\end{align}
where
\begin{align}
 E_{\phi,n}(\vec\gamma, g) = \frac{1}{n(n-1)}\sum_{1 \le i \ne j \le n} \sum_{k = 1}^c \gamma_{ik} g(v_j, k)  - \frac{1}{n}\sum_{i = 1}^n \sum_{k = 1}^c \gamma_{ik} \phi^*(g(v_i, k)) + \frac{1}{2a_n}\norm{g}_{\dom{H}}^2.
\end{align}
From the definition of $f$, we have $ I[f(x_i) = j] = I[\max_{k \ne j}\theta(x_i, k) - \theta(x_i, j) \le 0]$.
Let $\Delta_c = \cbrace{ \vec{x} \in \RealSet^c | x_i \ge 0 ~ \forall i, \sum_{i=1}^c x_i = 1 }$. Then, we relax the indicator function as follows:
\begin{align}
 \min_{f \in \fset{F}, \vec\gamma \in \Delta^n_c} ~ R_n(f) ~
 \subto ~ \max_{k \ne j} \theta(x_i, k) - \theta(x_i, j) \le -\gamma_{ij},
   ~ \max_{g \in \fset{G}} E_{\phi,n}(\vec\gamma, g) \le \eta. \label{eq:opt-sol}
\end{align}
This optimization problem is convex, and its solution is equivalent to \cref{eq:fal}. We prove this claims by the following corollary and theorem.
\begin{corollary}\label{lem:convex}
 If \cref{asm:hypothesis} holds, and $\ell$ is convex with respect to $\theta \in \Theta$, the optimization problem in \cref{eq:opt-sol} is a convex optimization problem.
\end{corollary}
\begin{theorem}\label{thm:opt-equiv}
 The solution of the optimization problem in \cref{eq:opt-sol} is equivalent to the solution of \cref{eq:opt-eq}.
\end{theorem}
The proofs of the corollary and the theorem can be found in \cref{sec:proofs}.

\section{Generalization Error Analysis}
\label{sec:gen-err}
We consider the generalization error bound of the learned hypothesis $f_n \in \fset{F}$ that is obtained by the algorithm described in \cref{sec:fal-div-est}. In our analysis, we use the two type of the {\em Rademacher complexity}, which measures the complexity of the class of the functions $f:\dom{Z}\to\RealSet$ and are defined as
\begin{align}
  \Rad_n(\fset{F}) = \Mean\bracket*{ \sup_{f \in \fset{F}}\frac{1}{n}\sum_{i=1}^n \sigma_i f(Z_i)}, \quad \Rad^{\rm abs}_n(\fset{F}) = \Mean\bracket*{ \sup_{f \in \fset{F}}\frac{1}{n}\abs*{\sum_{i=1}^n \sigma_i f(Z_i)}},
\end{align}
where $(\sigma_i)$ are the independent Rademacher variables, that is, $\p(\sigma_i = +1) = \p(\sigma_i = -1) = 1/2$.
\if0
\citet{bartlett2005local} proved the following theorem for the generalization error bound based on Bousquet's inequality:
\begin{theorem}[\citet{bartlett2005local}]\label{thm:gen-bound}
  Let $\fset{H} = \cbrace{h: \domprod{YX}\to\RealSet | h(Y,X) = \ell(Y,f(X)) - \ell(Y,f^*(X)),  f \in \fset{F}}$, and let $f_n \in \fset{F}$ be a hypothesis learned from the obtained set of samples $\set{S}_n$. Suppose that $h(y,x) - h(y',x') \le c$ for any $h \in \fset{H}$, $y,y' \in \dom{Y}$ and $x, x' \in \dom{X}$. Then, with probability at least $1-e^{-t}$
  \begin{align}
    R(f_n) - R(f^*) \le R_n(f_n) - R_n(f^*) + 4\Rad_n(\fset{H}) + \sigma_{\fset{H}}\sqrt{\frac{2t}{n}} + c\frac{4t}{3n},
  \end{align}
  where $\sigma^2_{\fset{H}} = \sup_{h \in \fset{H}}\Var[h(Y,X)]$.
\end{theorem}
\fi

In the generalization error analysis, since our fairness-aware learning algorithm have the probabilistic error $c\sqrt{2\tau/n}$, we consider the set of hypotheses defined as
\begin{align}
  \fset{F}_\tau = \cbrace{f \in \fset{F} | D_\phi(f) \le \eta + c\sqrt{2\tau/n}}.
\end{align}
where $c$ is defined as in \cref{thm:div-bound}. \cref{thm:div-bound} shows that with probability at least $1-e^{-\tau}$, $f_n \in \fset{F}_\tau$.
Hence, application of the theorem in \citep{bartlett2005local}, which is appeared in \cref{sec:gen-bartlett}, yields the generalization error bound for our algorithm.
\begin{corollary}
  Let $f^*_\tau$ be a hypothesis such that $R(f^*_\tau) = \inf_{f \in \fset{F}_\tau}R(f)$. Let $\fset{H}_\tau = \cbrace{h: \domprod{YX}\to\RealSet | h(Y,X) = \ell(Y,f(X)) - \ell(Y,f^*(X)),  f \in \fset{F}_\tau}$, and let $f_n \in \fset{F}$ be a hypothesis learned from the obtained set of samples $\set{S}_n$. Suppose that $h(y,x) - h(y',x') \le c$ for any $h \in \fset{H}_\tau$, $y,y' \in \dom{Y}$, and $x, x' \in \dom{X}$. Then, with probability at least $1-e^{-t}-e^{-\tau}$,
  \begin{align}
    R(f_n) - R(f_\tau^*) \le R_n(f_n) - R_n(f_\tau^*) + 4\Rad_n(\fset{H}_\tau) + \sigma_{\fset{H}_\tau}\sqrt{\frac{2t}{n}} + c\frac{4t}{3n}.
  \end{align}
\end{corollary}
Since $\fset{H}_\tau \subseteq \fset{H}$, we have $\Rad_n(\fset{H}_\tau) \le \Rad_n(\fset{H}), ~ \sigma_{\fset{H}_\tau} \le \sigma_{\fset{H}}$. Therefore, the convergence rate of the algorithm constrained by the $f$-divergences is lower than that of the algorithm without the constraint.

While our algorithm guarantees an upper bound on the $f$-divergences, it reduces the classification performance, as compared to the classifier learned by ERM. Accordingly, let us consider the generalization error of the optimal hypotheses with and without the restriction on the $f$-divergences:
\begin{align}
  R(f^*_\tau) - R(f^*). \label{eq:err-rest}
\end{align}
This error represents the reduction in the classification performance caused by restricting the $f$-divergences. Since the error cannot be directly evaluated, we define the estimator of \cref{eq:err-rest} as
\begin{align}
 R_n(f_n) - \inf_{f \in \fset{F}}R_n(f).
\end{align}
Our interest is to derive the convergence rate of this estimator.
\if0
To this end, we use the measure of the complexity for the class of the functions similar to the Rademacher complexity defined as
\begin{align}
  \Rad^{\rm abs}_n(\fset{F}) = \Mean\bracket*{ \sup_{f \in \fset{F}}\frac{1}{n}\abs*{\sum_{i=1}^n \sigma_i f(Z_i)}},
\end{align}
where $(\sigma_i)$ are the independent Rademacher variables.
\fi
We denote $\ell\circ\fset{F}' = \cbrace{ h:\domprod{XY}\to\RealSet | h(Y,X) = \ell(Y,f(X)), f \in \fset{F}'}$ for any $\fset{F}' \subseteq \fset{F}$, then the following theorem shows the convergence rate of the estimator.
\begin{theorem}\label{thm:rest-error-bound}
  Suppose that  $\ell(y,f(x)) - \ell(y',f(x')) \le c$ for any $f \in \fset{F}$ and $(y,x),(y',x') \in \dom{Z}$. Then, with probability at least $1-2e^{-t}-e^{-\tau}$,
  \begin{multline}
    \abs*{R(f^*_\tau) - R(f^*) - \paren*{R_n(f_n) - \inf_{f \in \fset{F}}R_n(f)} } \\ \le 4\Rad^{\rm abs}_n(\ell\circ\fset{F}_\tau) + 4\Rad^{\rm abs}_n(\ell\circ\fset{F}) + (\sigma_{\ell\circ\fset{F}_\tau} + \sigma_{\ell\circ\fset{F}})\sqrt{\frac{2t}{n}} + c\frac{8t}{3n}.
  \end{multline}
\end{theorem}
The proof of this theorem is appeared in \cref{sec:proofs}.

\section{Conclusions}
In this paper, we considered fairness-aware learning for a classification problem, with the aim of learning the classifier that returns the prediction with the lowest misclassification rate and the lowest dependence on the viewpoint.
Our contributions are as follows:
(1) We propose a novel generalized procedure for estimating the $f$-divergences for fairness-aware learning. Our generalized estimation procedure provides a tighter upper bound of the estimation error by introducing the maximum mean discrepancy.
(2) We formulate a general ERM framework for fairness-aware learning algorithm that is based on the empirical estimation procedure of the $f$-divergences, and that can guarantee an upper bound on the generalization dependency. Furthermore, we provide an analysis of the generalization error of the proposed fairness-aware learning algorithm.

%\subsubsection*{Acknowledgments}

\bibliographystyle{plainnat}
\bibliography{./reference}

\appendix

\section{Maximum Mean Discrepancy with Functions on a Reproducing Kernel Hilbert Space}
\label{sec:mmd-rkhs}

Since we need to solve the maximization problem in $D_{{\rm MMD}, f, n}(r)$, evaluation of the empirical MMD takes considerable cost causing use of the iterative algorithm.
However, if the elements of the class of the functions $\fset{G}$ are represented by the inner products of the parameters, which includes the functions in a reproducing kernel Hilbert space~(RKHS), the empirical MMD can be efficiently calculated. Let $k:\domprod{VYVY}\to\RealSet$ be a universal kernel, and let $\dom{H}$ be the RKHS induced by $k$. Let $\Phi:\domprod{VY}\to\dom{H}$ be the canonical feature map induced by $k$.
\begin{corollary}\label{cor:mmd-kernel}
Suppose that $\fset{G} = \cbrace{g | g(V,f(X)) = \abrace{\beta,\Phi(V,f(X))}, \norm{\beta}_{\dom{H}} \le 1}$. Then, the empirical MMD is equivalent to
\begin{align}
   D_{{\rm MMD}, f, n}(r) = \norm*{\frac{1}{n(n-1)}\sum_{1 \le i \ne j \le n}\bracket*{  \Phi(V_i, f(X_j)) - r(V_i, f(X_i)) \Phi(V_i, f(X_i))  }}_{\dom{H}}. \label{eq:mmd-kernel}
\end{align}
\end{corollary}
%The proof can be found in the supplemental material.
\begin{proof}[Proof of \cref{cor:mmd-kernel}]
  From the definition of the MMD and $\fset{G}$, we have
  \begin{align}
    D_{{\rm MMD}, f, n}(r) =& \sup_{\beta | \norm{\beta}_{\dom{H}} \le 1} \frac{1}{n(n-1)}\sum_{1 \le i \ne j \le n}\bracket*{  \abrace{\beta,\Phi(V_i, f(X_j))}_{\dom{H}} - r(V_i, f(X_i)) \abrace{\beta,\Phi(V_i, f(X_i))}_{\dom{H}}  } \\
    =& \sup_{\beta | \norm{\beta}_{\dom{H}} \le 1} \abrace*{\beta, \frac{1}{n(n-1)}\sum_{1 \le i \ne j \le n}\bracket*{  \Phi(V_i, f(X_j)) - r(V_i, f(X_i)) \Phi(V_i, f(X_i))  }}_{\dom{H}}.
  \end{align}
  Since the supremum is achieved if the direction of $\beta$ is equivalent to that of $\tfrac{1}{n(n-1)}\sum_{1 \le i \ne j \le n}\bracket{  \Phi(V_i, f(X_j)) - r(V_i, f(X_i)) \Phi(V_i, f(X_i))}$, we get the claim.
\end{proof}
For simplicity of notation, we let $\Phi(v_i, f(x_j))$ be represented by $\Phi_{ij}$, $r(v_i, f(x_i))$ by $r_i$, and $k(v_i, f(x_j), v_k, f(x_\ell)) = \abrace{\Phi(v_i, f(x_j)), \Phi(v_k, f(x_\ell))}$ by $k_{ijk\ell}$. Then, \cref{eq:mmd-kernel} can be rearranged as
\begin{align}
  \frac{1}{n^2}\sum_{1\le i,j \le n} r_i r_j k_{iijj} - \frac{2}{n}\sum_{i=1}^n r_i \paren*{\frac{1}{n(n-1)}\sum_{1 \le j \ne k \le n}k_{iijk}} + \frac{1}{n^2(n-1)^2}\sum_{1 \le i \ne j, k \ne \ell \le n} k_{ijk\ell}. \label{eq:mmd-kernel-short}
\end{align}
Let $\mat{Q}$ be a matrix such that $Q_{ij} = k_{iijj}$, and let $\vec{p}$ be a vector such that $p_i = \tfrac{1}{n-1}\sum_{i \le j \ne k \le n}k_{iijk}$. Let $\vec{r}$ be a vector representation of $r_i$. The matrix representation of the minimization of \cref{eq:mmd-kernel-short} is obtained as
\begin{align}
  \min_{\vec{r} \ge 0} ~ \frac{1}{2}\vec{r}^T\mat{Q}\vec{r} - \vec{p}^T\vec{r}.
\end{align}
The minimizer of \cref{eq:mmd-kernel-short} with respect to $r$ can be easily obtained if $\mat{Q}$ is a positive definite matrix.

\section{Generalization Error Bound of \citet{bartlett2005local}}
\label{sec:gen-bartlett}
\citet{bartlett2005local} proved the following theorem for the generalization error bound based on Bousquet's inequality:
\begin{theorem}[\citet{bartlett2005local}]\label{thm:gen-bound}
  Let $\fset{H} = \cbrace{h: \domprod{YX}\to\RealSet | h(Y,X) = \ell(Y,f(X)) - \ell(Y,f^*(X)),  f \in \fset{F}}$, and let $f_n \in \fset{F}$ be a hypothesis learned from the obtained set of samples $\set{S}_n$. Suppose that $h(y,x) - h(y',x') \le c$ for any $h \in \fset{H}$, $y,y' \in \dom{Y}$ and $x, x' \in \dom{X}$. Then, with probability at least $1-e^{-t}$
  \begin{align}
    R(f_n) - R(f^*) \le R_n(f_n) - R_n(f^*) + 4\Rad_n(\fset{H}) + \sigma_{\fset{H}}\sqrt{\frac{2t}{n}} + c\frac{4t}{3n},
  \end{align}
  where $\sigma^2_{\fset{H}} = \sup_{h \in \fset{H}}\Var[h(Y,X)]$.
\end{theorem}

\section{Proofs}
\label{sec:proofs}

\subsection{Proof of \cref{thm:div-bound}}
In order to prove \cref{thm:div-bound}, we prove following lemmas.
\begin{lemma}\label{lem:as-bound-u-stat}
 Suppose that $c_\ell \le r^*(V, f(X)) \le c_u$ almost surely where $c_\ell \in (0, 1]$ and $c_u \in [1, \infty)$, then
 \begin{multline}
  \partial\phi(c_\ell) - \partial\phi(c_u) \le \\ \frac{1}{2}\paren[\Big]{\partial\phi(r^*(V, f(X)))r^*(V, f(X)) + \partial\phi(r^*(V', f(X')))r^*(V', f(X')) \\ - \partial\phi(r^*(V', f(X))) - \partial\phi(r^*(V, f(X')))} \\ \le \partial\phi(c_u)c_u - \partial\phi(c_\ell) ~ a.s. .
 \end{multline}
\end{lemma}
\begin{proof}
Since $\partial\phi$ is non-decreasing function due to the convexity of $\phi$, we have
\begin{align}
 \partial\phi(c_\ell) \le \partial\phi(r^*(V', f(X))), \partial\phi(r^*(V, f(X))) \le \partial\phi(c_u) ~ a.s. . \label{eq:bound-subdif-phi}
\end{align}
From the assumption that the subdifferential of $\phi$ contains zero, $\partial\phi(u) \le 0$ for $u \in (0,1]$ and $\partial\phi(u) \ge 0$ for $u \in [1,\infty)$ which results that $\partial\phi(c_u)$ and $\partial\phi(c_\ell)$ are positive and negative, respectively. By this fact and \cref{eq:bound-subdif-phi}, we have
\begin{align}
 \partial\phi(c_\ell) \le \partial\phi(r^*(V, f(X)))r^*(V, f(X)) \le \partial\phi(c_u)c_u ~ a.s. .\label{eq:bound-prod-subphi}
\end{align}
Combining \cref{eq:bound-subdif-phi,eq:bound-prod-subphi} gives the claim.
\end{proof}
\begin{lemma}\label{lem:as-bound-div}
 Suppose that $c_\ell \le r^*(V, f(X)) \le c_u$ almost surely where $c_\ell \in (0, 1]$ and $c_u \in [1, \infty)$, then
\begin{align}
  -\max\cbrace{\phi(c_\ell), \phi(c_u)} \le \Mean[\phi(r^*(V, f(X)))] - \phi(r^*(V, f(X))) \le \max\cbrace{\phi(c_\ell), \phi(c_u)} ~ a.s. .
\end{align}
\end{lemma}
\begin{proof}
From the assumption that the subdifferential of $\phi$ contains zero, $\partial\phi(u) \le 0$ for $u \in (0,1]$ and $\partial\phi(u) \ge 0$ for $u \in [1,\infty)$ which yields that $\phi(u)$ is non-increasing in $(0,1]$ and non-decreasing in $[1,\infty)$. Therefore, $0 \le \phi(r^*(V, f(X))) \le \max\cbrace{\phi(c_u),\phi(c_\ell)} ~ a.s.$, which gives the claim.
\end{proof}
\begin{proof}[Proof of \cref{thm:div-bound}]
The error $D_\phi(f) - D_{\phi,n}(f)$ is decomposed as
\begin{align}
  D_\phi(f) - D_{\phi,n}(f) = D_\phi^{r^*}(f) - D_{\phi,n}^{r^*}(f) + D_{\phi,n}^{r^*}(f) - D_{\phi,n}^{r_n}(f). \label{eq:error-decomp}
\end{align}
From the definition of the subdifferential, we have
\begin{align}
  & D_{\phi,n}^{r^*}(f) - D_{\phi,n}^{r_n}(f) \\
   =& \frac{1}{n}\sum_{i = 1}^n \bracket{\phi(r^*(V_i, f(X_i))) - \phi(r_n(V_i, f(X_i)))} \\
   \le& \frac{1}{n}\sum_{i = 1}^n \partial\phi(r^*(V_i, f(X_i)))\paren*{ r^*(V_i, f(X_i)) - r_n(V_i, f(X_i))} \\
   =& \frac{1}{n}\sum_{i = 1}^n \partial\phi(r^*(V_i, f(X_i)))r^*(V_i, f(X_i)) - \frac{1}{n}\sum_{i = 1}^n \partial\phi(r^*(V_i, f(X_i))) r_n(V_i, f(X_i)). \label{eq:emp-error}
\end{align}
Since the left term in \cref{eq:emp-error} is regarded as the empirical mean of $\partial\phi(r^*(V_i, f(X_i)))$ with respect to $\pjoin{V, f(X)}$ weighted by $r^*(V, f(X)) = d\pidep{V,f(X)}/d\pjoin{V,f(X)}$, it is well approximated by the U-statistics of the expectation of $\partial\phi(r^*(V_i, f(X_i)))$ with respect to $\pidep{V, f(X)}$. We thus decompose the right hand side in \cref{eq:emp-error} using the U-statistics of $\partial\phi(r^*(V_i, f(X_i)))$ as
\begin{align}
  & D_{\phi,n}^{r^*}(f) - D_{\phi,n}^{r_n}(f) \\
   \le& \frac{1}{n}\sum_{i = 1}^n \partial\phi(r^*(V_i, f(X_i)))r^*(V_i, f(X_i)) - \frac{1}{n(n-1)}\sum_{1 \le i \ne j \le n}\partial\phi(r^*(V_i, f(X_j))) \\ & + \frac{1}{n(n-1)}\sum_{1 \le i \ne j \le n}\bracket*{ \partial\phi(r^*(V_i, f(X_j))) - r_n(V_i, f(X_i))\partial\phi(r^*(V_i, f(X_i))) }. \label{eq:emp-error-u}
\end{align}
Since $\fset{G}$ contains $\partial\phi(r^*(V, f(X)))/a$, the last term in \cref{eq:emp-error-u} is bounded above by the MMD as
\begin{align}
  &\frac{1}{n(n-1)}\sum_{1 \le i \ne j \le n}\bracket*{ \partial\phi(r^*(V_i, f(X_j))) - r_n(V_i, f(X_i))\partial\phi(r^*(V_i, f(X_i))) } \\
   \le& a\sup_{g \in \fset{G}} \frac{1}{n(n-1)}\sum_{1 \le i \ne j \le n}\bracket*{ g(V_i, f(X_j)) - r_n(V_i, f(X_i))g(V_i, f(X_i)) } \\
   =& aD_{{\rm MMD}, f, n}(r_n).
\end{align}
Letting the first two terms in \cref{eq:emp-error-u}  be
\begin{align}
  U_{\phi,n}(f) =& \frac{1}{n}\sum_{i = 1}^n \partial\phi(r^*(V_i, f(X_i)))r^*(V_i, f(X_i)) - \frac{1}{n(n-1)}\sum_{1 \le i \ne j \le n}\partial\phi(r^*(V_i, f(X_j))) ,
\end{align}
the error is bounded above as
\begin{align}
  D_\phi(f) - D_{\phi,n}(f) \le D_\phi^{r^*}(f) - D_{\phi,n}^{r^*}(f) + U_{\phi,n}(f) + aD_{{\rm MMD}, f, n}(r_n).
\end{align}
Next, we derive the probabilistic bound on the $D_\phi^{r^*}(f) - D_{\phi,n}^{r^*}(f) + U_{\phi,n}(f)$. The expectations of $U_{\phi,n}(f)$ is equivalent to zero
\begin{align}
  \Mean[U_{\phi,n}(f)] =& \Mean\bracket*{ \frac{1}{n}\sum_{i = 1}^n \partial\phi(r^*(V_i, f(X_i)))r^*(V_i, f(X_i)) - \frac{1}{n(n-1)}\sum_{1 \le i \ne j \le n}\partial\phi(r^*(V_i, f(X_j))) } \\
   =& \Mean\bracket{\partial\phi(r^*(V, f(X)))r^*(V, f(X))} - \Mean\bracket{\partial\phi(r^*(V', f(X)))} \\
   =& \Mean\bracket{\partial\phi(r^*(V', f(X)))} - \Mean\bracket{\partial\phi(r^*(V', f(X)))} \\
   =& 0. \label{eq:exp-u}
\end{align}
As proved the almost surely bound in \cref{lem:as-bound-u-stat,lem:as-bound-div}, application of the exponential inequality for the U-statistics~\cite{hoeffding:1963} gives with probability at least $1-e^{-t}$
\begin{align}
  D_\phi^{r^*}(f) - D_{\phi,n}^{r^*}(f) + U_{\phi,n}(f) \le
   \Mean\bracket{ D_\phi^{r^*}(f) - D_{\phi,n}^{r^*}(f) + U_{\phi,n}(f) } + c\sqrt{\frac{2t}{n}},
\end{align}
where $c$ is a constant defined as in the claim.
As shown in \cref{eq:exp-u}, $\Mean[U_{\phi,n}(f) ] = 0$. In addition, since $D_{\phi,n}^{r^*}(f)$ is the unbiased estimator of $D_\phi^{r^*}(f)$, $\Mean[D_\phi^{r^*}(f)] = D_\phi^{r^*}(f)$.
\end{proof}

\subsection{Proof of \cref{thm:opt-equiv} and \cref{lem:convex}}
\begin{proof}[Proof of \cref{thm:opt-equiv}]
 For any $i$, if $\max_{k \ne j} \theta(x_i, k) - \theta(x_i, j) \le 0$ holds for any $j \in \iset{I}_i \subseteq \cbrace{1,...,c}$ such that $\abs{\iset{I}_i} > 1$, $\max_{k \ne j} \theta(x_i, k) - \theta(x_i, j)$ is positive for all $j$ because of the definition of $\max$. Thus, this case violates the first constraint in \cref{eq:opt-sol} since $\sum_{j=1}^c \gamma_{ij} = 1$ and $\gamma_{ij} \ge 0$ require $\gamma_{ij} > 0$ for some $j$.
 Therefore, for any $i$ and $j$, if $\max_{k \ne j} \theta(x_i, k) - \theta(x_i, j) \le 0$, then $\max_{k \ne j} \theta(x_i, k) - \theta(x_i, p) > 0$ for any $p \ne j$ because of the definition of $\max$. This indicates that if $\vec\gamma$ are feasible, then only one element of $\cbrace{\gamma_{ij}}_{j = 1}^c$ is one and the others are zero for each $i$. Since if $\max_{k \ne j} \theta(x_i, k) - \theta(x_i, j) \le 0$ holds then $f(x_i) = j$, $I[f(x_i) = j] = \gamma_{ij}$. Thus, since the solution of \cref{eq:opt-sol} holds the constraints in \cref{eq:opt-eq}, we get the claim.
\end{proof}

\begin{proof}[Proof of \cref{lem:convex}]
 Since $E_{\phi,n}(\vec\gamma, g)$ is linear with respect to $\vec\gamma$, $\max_{g \in \fset{G}} E_{\phi,n}(\vec\gamma, g)$ is a convex function with respect to $\vec\gamma$. Since $R_n(f)$ is convex with respect to $\theta$ because of the convexity of $\ell$, the objective function in \cref{eq:opt-sol} is convex with respect to $f$ and $\vec\gamma$. In addition, $\max_{k \ne j} \theta(x_i, k) - \theta(x_i, j) + \gamma_{ij}$ is convex with respect to $\theta$ and $\gamma_{ij}$. Thus, all constraints are convex inequalities or linear equations.
\end{proof}

\subsection{Proof of \cref{thm:rest-error-bound}}
For the proof of \cref{thm:rest-error-bound}, we prove the following theorem.
\begin{theorem}\label{thm:conv-bound}
 Let $f^* \in \fset{F}$ be a hypothesis such that $R(f^*) = \inf_{f \in \fset{F}}R(f)$. Suppose that $\ell(y,f(x)) - \ell(y',f(x')) \le c$ for any $f \in \fset{F}$ and $(y,x),(y',x') \in \dom{Z}$. Then, with probability at least $1 - e^{-t}$
 \begin{align}
  \abs*{R(f^*) - \inf_{f \in \fset{F}}R_n(f)} \le 4\Rad^{\rm abs}_n(\ell\circ\fset{F}) + \sigma_{\ell\circ\fset{F}}\sqrt{\frac{2t}{n}} + c\frac{4t}{3n}.
 \end{align}
\end{theorem}
\begin{proof}
 Since $\sup_{f \in \fset{F}}(R(f^*) - R(f)) = \inf_{f \in \fset{F}}R(f) - \inf_{f \in \fset{F}}R(f) = 0$, we have
 \begin{align}
   R(f^*) - \inf_{f \in \fset{F}}R_n(f) = \sup_{f \in \fset{F}}\paren*{R(f^*) - R_n(f)} \le& \sup_{f \in \fset{F}}(R(f^*) - R(f)) + \sup_{f \in \fset{F}}(R(f) - R_n(f)) \\
  =& \sup_{f \in \fset{F}}(R(f) - R_n(f)).
 \end{align}
 From the definition of $\sup$, we have
 \begin{align}
  \sup_{f \in \fset{F}}\paren*{R(f^*) - R_n(f)} \ge R(f^*) - R_n(f^*).
 \end{align}
 Hence, we have
 \begin{multline}
  \abs*{\sup_{f \in \fset{F}}\paren*{R(f^*) - R_n(f)}} \le\\ \max\cbrace*{\abs*{\sup_{f \in \fset{F}}(R(f) - R_n(f))},\abs*{R(f^*) - R_n(f^*)}} \le \sup_{f \in \fset{F}}\abs*{R(f) - R_n(f)}.
 \end{multline}
 The bound on $\sup_{f \in \fset{F}}\abs{R(f) - R_n(f)}$ is derived in the same manner of the proof of \cref{thm:gen-bound}. Application of the Bousquet's inequality gives with probability at least $1-e^{-t}$
 \begin{align}
  \sup_{f \in \fset{F}}\abs{R(f) - R_n(f)} \le \Mean\sup_{f \in \fset{F}}\abs{R(f) - R_n(f)} + \sqrt{\frac{2v}{n}} + c\frac{t}{3n},
 \end{align}
 where $v = 2\Mean\sup_{f \in \fset{F}}\abs{R(f) - R_n(f)} + \sigma^2_{\ell\circ\fset{F}}$. By the fact that $\sqrt{u + v} \le \sqrt{u} + \sqrt{v}$ and $2\sqrt{uv} \le \alpha u + v/\alpha$ for $u,v,\alpha \ge 0$, application of the symmetrization technique yields the claim.
\end{proof}
\begin{corollary}\label{cor:conv-tau-bound}
  Let $f^*_\tau \in \fset{F}_\tau$ be a hypothesis such that $R(f^*) = \inf_{f \in \fset{F}_\tau}R(f)$. Suppose that $\ell(y,f(x)) - \ell(y',f(x')) \le c$ for any $f \in \fset{F}$ and $(y,x),(y',x') \in \dom{Z}$. Then, with probability at least $1 - e^{-t} - e^{-\tau}$
  \begin{align}
   \abs*{R(f^*_\tau) - R_n(f_n)} \le 4\Rad^{\rm abs}_n(\ell\circ\fset{F}_\tau) + \sigma_{\ell\circ\fset{F}_\tau}\sqrt{\frac{2t}{n}} + c\frac{4t}{3n}.
  \end{align}
\end{corollary}
\begin{proof}
  The proof follows the same manner of the proof of \cref{thm:conv-bound} expect the upper bound of $R(f^*) - R_n(f_n)$. From \cref{thm:div-bound}, we have with probability at least $1-e^{-\tau}$
  \begin{align}
    R(f^*) - R_n(f_n) \le \sup_{f \in \fset{F}_\tau}(R(f^*) - R_n(f)) \le \sup_{f \in \fset{F}_\tau}(R(f) - R_n(f)).
  \end{align}
\end{proof}
\begin{proof}[Proof of \cref{thm:rest-error-bound}]
  The error is bounded above as
\begin{align}
     \abs*{R(f^*_\tau) - R(f^*) - \paren*{R_n(f_n) - \inf_{f \in \fset{F}}R_n(f)} }
     \le \abs*{R(f^*_\tau) - R_n(f_n)} + \abs*{R(f^*) - \inf_{f \in \fset{F}}R_n(f)}. \label{eq:decomp-rest-err}
\end{align}
Combining \cref{thm:conv-bound,cor:conv-tau-bound} gives the claim.
\end{proof}

\end{document}